\documentclass[10pt,twocolumn,letterpaper]{article}

\usepackage{iccv}              
\usepackage{amsthm}
\usepackage{amsmath}        
\usepackage{amssymb}        
\usepackage{graphicx}       
\usepackage{algorithm}
\usepackage{algorithmic}
\usepackage{pifont}
\usepackage{multirow}
\usepackage[accsupp]{axessibility}
\newcommand{\cmark}{\text{\ding{51}}}
\newcommand{\xmark}{\text{\ding{55}}}

\usepackage{subcaption}

\usepackage{colortbl}
\usepackage{tcolorbox} 
\definecolor{matching}{RGB}{230, 240, 255} 
\definecolor{others}{RGB}{255, 235, 235} 
\definecolor{woDP}{RGB}{235, 255, 235} 

\makeatletter
\renewcommand{\paragraph}{%
  \@startsection{paragraph}{4}{\z@}{1pt}{-5pt}{\normalfont\normalsize\bfseries}%
}
\makeatother

\makeatletter
\renewcommand\@makefntext[1]{%
    \noindent\makebox[0.75em][r]{\@thefnmark.}\hspace{0.5em}#1%
}
\makeatother

\makeatletter
\newtheorem*{rep@theorem}{\rep@title}
\newcommand{\newreptheorem}[2]{%
\newenvironment{rep#1}[1]{%
 \def\rep@title{#2 \ref{##1}}%
 \begin{rep@theorem}}%
 {\end{rep@theorem}}}
\makeatother

\usepackage{amsmath,amsfonts,bm}

\def\eqref#1{equation~\ref{#1}}

\def\1{\bm{1}}

\def\vmu{{\bm{\mu}}}

\def\vp{{\bm{p}}}

\def\vr{{\bm{r}}}

\def\vu{{\bm{u}}}
\def\vv{{\bm{v}}}

\def\vx{{\bm{x}}}
\def\vy{{\bm{y}}}
\def\vz{{\bm{z}}}

\def\vomega{{\bm{\omega}}}
\def\veta{{\bm{\eta}}}

\def\mP{{\bm{P}}}

\def\mZ{{\bm{Z}}}

\DeclareMathAlphabet{\mathsfit}{\encodingdefault}{\sfdefault}{m}{sl}
\SetMathAlphabet{\mathsfit}{bold}{\encodingdefault}{\sfdefault}{bx}{n}

\def\gB{{\mathcal{B}}}

\def\gD{{\mathcal{D}}}

\def\gF{{\mathcal{F}}}

\def\gM{{\mathcal{M}}}

\def\gS{{\mathcal{S}}}

\def\gV{{\mathcal{V}}}

\def\gX{{\mathcal{X}}}

\def\gZ{{\mathcal{Z}}}

\definecolor{iccvblue}{rgb}{0.21,0.49,0.74}
\usepackage[pagebackref,breaklinks,colorlinks,allcolors=iccvblue]{hyperref}

\def\paperID{13571} 
\def\confName{ICCV}
\def\confYear{2025}

\title{Improving Noise Efficiency in Privacy-preserving Dataset Distillation}

\newtheorem{assumption}{Assumption}
\newtheorem{theorem}{Theorem}
\newtheorem{lemma}{Lemma}

\newreptheorem{theorem}{Theorem}

\author{Runkai Zheng\\
Carnegie Mellon University\\
runkaiz@andrew.cmu.edu\\
\and
Vishnu Asutosh Dasu\\
Pennsylvania State University\\
vdasu@psu.edu\\
\and
Yinong Oliver Wang\\
Carnegie Mellon University\\
yinongwa@cs.cmu.edu\\
\and
Haohan Wang\\
University of Illinois Urbana-Champaign\\
haohanw@illinois.edu\\
\and
Fernando De la Torre\\
Carnegie Mellon University\\
ftorre@cs.cmu.edu}

\begin{document}
\maketitle
\begin{abstract}
Modern machine learning models heavily rely on large datasets that often include sensitive and private information, raising serious privacy concerns. Differentially private (DP) data generation offers a solution by creating synthetic datasets that limit the leakage of private information within a predefined privacy budget; however, it requires a substantial amount of data to achieve performance comparable to models trained on the original data. To mitigate the significant expense incurred with synthetic data generation, Dataset Distillation (DD) stands out for its remarkable training and storage efficiency. This efficiency is particularly advantageous when integrated with DP mechanisms, curating compact yet informative synthetic datasets without compromising privacy. However, current state-of-the-art private DD methods suffer from a synchronized sampling-optimization process and the dependency on noisy training signals from randomly initialized networks. This results in the inefficient utilization of private information due to the addition of excessive noise. To address these issues, we introduce a novel framework that decouples sampling from optimization for better convergence and improves signal quality by mitigating the impact of DP noise through matching in an informative subspace. On CIFAR-10, our method achieves a \textbf{10.0\%} improvement with 50 images per class and \textbf{8.3\%} increase with just \textbf{one-fifth} the distilled set size of previous state-of-the-art methods, demonstrating significant potential to advance privacy-preserving DD. \footnote{Source code is available at \hyperlink{https://github.com/humansensinglab/Dosser}{https://github.com/humansensinglab/Dosser}.}

\end{abstract}    
\vspace{-9pt}
\section{Introduction}
\label{sec:intro}

\begin{figure}[ht]
    \centering
    \subfloat[]{%
        \includegraphics[width=0.48\linewidth]{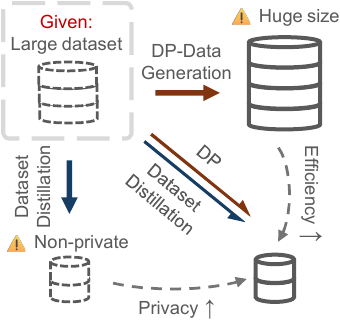}%
        \label{fig:intro-a}}%
    \hfill
    \subfloat[]{%
        \includegraphics[width=0.48\linewidth]{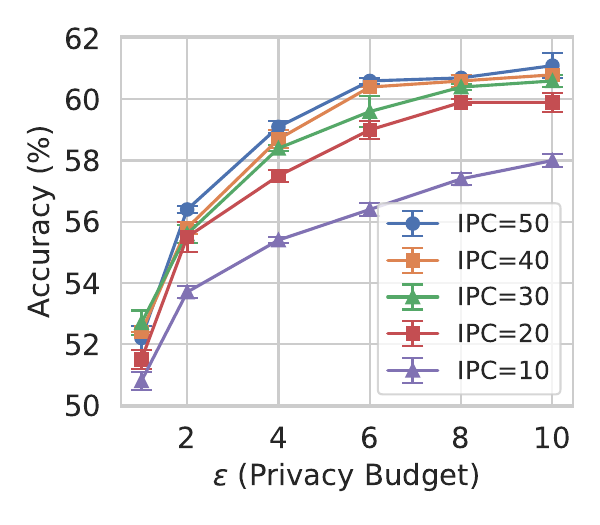}%
        \label{fig:intro-b}}
    \caption{
        (a) Overview of private dataset distillation. 
        (b) Accuracy of distilled CIFAR-10 images across privacy budgets and IPC.
    }
    \label{fig:intro}
\end{figure}

In modern machine learning, large datasets are essential for training robust and accurate models. However, they often contain sensitive information, posing challenges for data sharing and privacy protection. Differentially private (DP) data generation addresses this by producing synthetic datasets within a controlled privacy budget, typically using mechanisms like Differentially Private Stochastic Gradient Descent (DP-SGD)~\cite{abadi2016deep}. These approaches aim to approximate the data distribution and generate privacy‑preserving samples. Despite notable advances, balancing utility and privacy remains difficult. Moreover, generative models often require synthesizing and storing large volumes of data to match real‑data performance. For instance, \citet{ghalebikesabi2023differentially} generates data 20x larger than the original dataset yet still underperforms it. Such volume demands lead to high storage and computational costs during training.

Dataset Distillation (DD) \cite{wang2018dataset} has emerged as a promising alternative that addresses some of the inherent limitations of generative models. As fewer but highly informative synthetic samples can be retained for comparable downstream performance, DD ensures that models trained on distilled datasets perform similarly to those trained on larger original datasets with a significantly reduced storage.

While DD effectively minimizes dataset size and becomes visually anonymized, it does not inherently provide privacy guarantees. Conversely, differentially private data generation ensures privacy but with large synthetic dataset sizes, which may not be storage or computation-efficient. This presents a critical motivation for our work: to develop a method that integrates the compactness of dataset distillation with the stringent privacy guarantees of differential privacy (\cref{fig:intro}-a). Achieving this integration is challenging, as both privacy preservation and dataset compactness tend to negatively impact the utility of the dataset. 

Integrating dataset distillation with differential privacy, as in methods like PSG~\cite{chen2022private} and NDPDC~\cite{zheng2022differentially}, enables private data synthesis by matching training signals (e.g., gradients, features) from the original dataset. However, these matching-based approaches face key limitations. First, they couple sampling and optimization, requiring each optimization step to be paired with new noisy queries, leading to degraded signal utility. Second, they rely on randomly initialized networks to extract training signals, which often capture uninformative details and yield low signal-to-noise ratios (SNR), amplifying DP noise effects. As a result, existing methods struggle to fully exploit limited private signals, resulting in suboptimal distilled dataset performance.

To address the limitations of matching-based DD under DP constraints, we propose a framework combining Decoupled Optimization and Sampling (DOS) with Subspace-based Error Reduction (SER) to better exploit information from private data. DOS first samples a fixed number of training signals under a DP budget, then optimizes the synthetic dataset using these precomputed signals over a separate number of iterations. Decoupling these stages allows flexible trade-offs: fewer sampling steps reduce cumulative noise, while sufficient optimization improves image quality. SER further boosts utility by projecting signals into an informative subspace learned from auxiliary data, where DP noise is injected. This concentrates signal power on high-utility dimensions, increasing the signal-to-noise ratio and mitigating DP degradation. Together, DOS and SER enhance noise efficiency, enabling compact synthetic datasets that better balance privacy and utility.

\section{Background and Related Works}
\label{sec:background}

\subsection{Differential Privacy}
\label{subsec:differential_privacy}

Differential Privacy (DP) \cite{dp} is a rigorous mathematical framework that quantifies the privacy guarantees of algorithms operating on sensitive data. A randomized mechanism $\gM: \gB \to \mathcal{R}$ with domain $\gB$ and range $\mathcal{R}$ satisfies $(\epsilon, \delta)$-differential privacy if, for any two adjacent datasets $B, B' \in \gB$ differing in at most one element, and for any subset of outputs $S \subseteq \mathcal{R}$, the following inequality holds:
\[
\Pr[\gM(B) \in S] \leq e^\epsilon \Pr[\gM(B') \in S] + \delta.
\]
where $ \epsilon \geq 0 $ and $ \delta \geq 0 $ are parameters that measure the strength of the privacy guarantee: the smaller they are, the stronger the privacy.
We enforce DP using the Gaussian Mechanism (GM), which perturbs a function $f: B \to \mathbb{R}^d$ by adding noise:  
$\text{GM}_{\sigma}(B) = f(B) + \mathcal{N}(0, \sigma^2 \mathbb{I}_d)$.  
Under the add/remove-one model, the $\ell_2$‑sensitivity of $f$ is bounded, allowing calibrated noise addition.
To track cumulative privacy loss, we adopt Rényi Differential Privacy (RDP)~\cite{mironov2017renyi} with privacy amplification via subsampling~\cite{wang2019subsampled}, and convert the results to standard $(\epsilon,\delta)$‑DP using composition rules. Finally, we apply the post-processing property~\cite{dp} to ensure that any downstream operations preserve the established privacy budget.

\subsection{Dataset Distillation}
\label{subsec:dataset_distillation}

In the context of large datasets, dataset distillation aims to reduce the dataset size while retaining the critical information needed to train a model effectively. We denote a data sample by $\vx$ and its label by $\vy$, focusing on classification problems where $g_{\theta}(\cdot)$ represents a model parameterized by $\theta$, and $\ell(g_{\theta}(\vx), \vy)$ denotes the cross-entropy loss between the model output $g_{\theta}(\vx)$ and the label $\vy$. Let $\gD$ and $\gZ$ denote the original and synthetic datasets, respectively. Our objective is to find a smaller dataset $\gZ$ such that training on $\gZ$ yields similar performance as training on $\gD$. Formally, we define this dataset distillation problem as:
\begin{align*}
    \arg \min_{\gZ} \; \mathbb{E}_{(x, y) \sim \gD} \, \ell(g_{\theta(\gZ)}(x), y),& \\
    \text{where} \; \theta(\gZ) = \arg \min_{\theta} \; \mathbb{E}_{(x, y) \sim \gZ} \, \ell(g_{\theta}(x), y), \;& |\gZ| \ll |\gD|
\end{align*}

Following the taxonomy developed by~\citet{sachdeva2023datadistillationsurvey}, we discuss various previous methods of tackling this problem.
\textbf{Meta-Model Matching}
involves an inner optimization step to update model parameters $\theta$ and an outer optimization step to refine $\gZ$, aiming to make $\gZ$ as informative as possible for training $\theta$. \citet{wang2018dataset} use stochastic gradient descent (SGD) for the inner loop and Truncated Back-propagation Through Time (TBPTT) to optimize the outer loop by unrolling a fixed number of inner loop steps. 
\textbf{Gradient matching methods}
focus on matching the gradients of the neural network parameters when trained on synthetic data to those when trained on the original data. For instance, \citet{zhao2021dataset} proposes optimizing synthetic data such that the gradients of a model trained on this data mimic those from the original dataset, effectively capturing essential training dynamics in a condensed form. Extensions like \cite{zhao2021differentiable} incorporate differentiable data augmentation to enhance diversity and robustness. \textbf{Trajectory matching methods} extend this idea by matching the entire training trajectory of the model parameters. \citet{cazenavette2022dataset} match the sequence of model states during training (the trajectory) when trained on synthetic data to those from the real data, capturing a more comprehensive view of the learning process. This approach ensures that the condensed dataset leads to similar model behavior throughout training, not just in immediate gradients. Subsequent works build upon these concepts by integrating contrastive signals \cite{lee2022dataset}, aligning loss curvature \cite{shin2023loss}, and scaling up to larger datasets \cite{cui2023scaling}, among others.
\textbf{Distribution matching methods}
focus on aligning feature distributions between synthetic and real datasets. \citet{wang2022cafe} propose aligning features in a latent space to improve condensation, and subsequent works minimize statistical discrepancies using metrics like Maximum Mean Discrepancy \cite{zhang2024m3d} or exploit attention mechanisms for efficient distillation \cite{sajedi2023datadam}. \citet{liu2023dataset} introduces Wasserstein distance as an alternative metric of distribution discrepancy to build a distribution matching framework.
\textbf{Kernel-based distillation methods}
leverage theoretical insights from kernel ridge regression and infinitely wide networks to distill datasets; foundational works like \cite{nguyen2021dataset, nguyen2021datasetinfinitely} utilize kernel methods for condensation, while later studies improve efficiency and scalability through neural feature regression \cite{zhou2022dataset} and random feature approximations \cite{loo2022efficient}. Other works continue to refine these approaches by incorporating implicit gradients and convex optimization techniques \cite{du2023minimizing, loo2023dataset}.
These various methodologies reflect the diverse approaches employed in dataset distillation, each contributing unique perspectives and techniques.

\subsection{Differentially Private Dataset Distillation}
\label{subsec:differentially_private_dd}

The integration of dataset distillation (DD) with differential privacy (DP) has received considerable attention in recent literature. A recent technique known as DP-KIP \cite{vinaroz2024differentially} utilizes DP-SGD to update synthetic data within the Kernel-Induced Points (KIP) framework, offering an effective approach for distilling private datasets. Another well-developed direction involves incorporating DP within matching-based methods, where calibrated Gaussian noise is added to the matching signal before computing matching metrics. For example, Private Set Generation (PSG) \cite{chen2022private} introduces Gaussian noise into clipped gradients for matching, while Non-linear Differentially Private Dataset Condensation (NDPDC) \cite{zheng2022differentially} applies Gaussian noise to clipped features extracted from randomly initialized networks. By the post-processing theorem, these matching-based methods ensure differential privacy by aligning DP-protected signals from private datasets.


However, current matching-based DP-DD methods often couple the process of sampling signals from the private dataset with the process of optimizing the distilled images. We argue that this coupling leads to unnecessary noise addition. When sampling and optimization are performed simultaneously, their iterations are forced to be equal. When a high number of optimization iterations is required for convergence, an equally large number of sampling steps is needed. These numerous sampling steps require excessive noise to maintain DP. Consequently, the trade-off between iteration count and noise magnitude limits the effectiveness of these methods, as they struggle to maximize signal utility from the private dataset within a fixed privacy budget.
Moreover, due to restricted access to the private training dataset, matching-based methods rely on randomly initialized neural networks to extract training signals from the private data, instead of a pre-trained network. However, randomly initialized networks capture numerous uninformative details, which lowers the signal-to-noise ratio (SNR) of the training signals. This low SNR amplifies the negative impact of added noise, further compromising the utility of the training signals and the performance of the distilled dataset.

\section{Methodologies}

Our approach maximizes the utility of training signals from two perspectives: first, we decouple the sampling process from the optimization process, allowing for extended optimization iterations without unnecessary noise addition; second, we introduce an auxiliary dataset via generative models to identify the most informative signal subspace within randomly initialized neural networks, enhancing the signal-to-noise ratio (SNR)\footnote{Noise here in SNR refers to the uninformative features captured by randomly initialized neural networks, not the noise added for DP guarantees} to reduce the impact of added noise.

\subsection{Preliminaries and Annotations}


\paragraph{Private dataset} Let $\mathcal{D} = { (\gX^{(c)}, y^{(c)}) }_{c=1}^C$ denote the private dataset, where: $C$ is the total number of classes. $\gX^{(c)} = \{ \vx_j^{(c)} \}_{j=1}^{N^{(c)}}$ is the set of images belonging to class $c$. $y^{(c)}$ is the label associated with class $c$. $N^{(c)} = \left| \gX^{(c)} \right|$ is the number of images per class (IPC). 

\paragraph{Synthetic Dataset} Our goal is to generate a synthetic dataset $\gZ = \{ \gZ^{(c)} \}_{c=1}^C$, where $\gZ^{(c)} = \{ \vz_j^{(c)} \}_{j=1}^M$ represents the set of synthetic images for class $c$, and $M$ is the number of synthetic IPC.

\paragraph{Training Signal} We consider various types of training signals in matching-based methods, such as features in distribution-matching methods~\cite{zheng2022differentially} and gradients in gradient-matching methods~\cite{chen2022private}. Our framework can be generalized to any type. The extraction of the signal is represented by a parameterized function $f_{\theta}$, and the signal for matching is denoted $\vv$ and $\vu$ for real and synthetic datasets.

\begin{figure*}
    \centering
    \includegraphics[width=\linewidth]{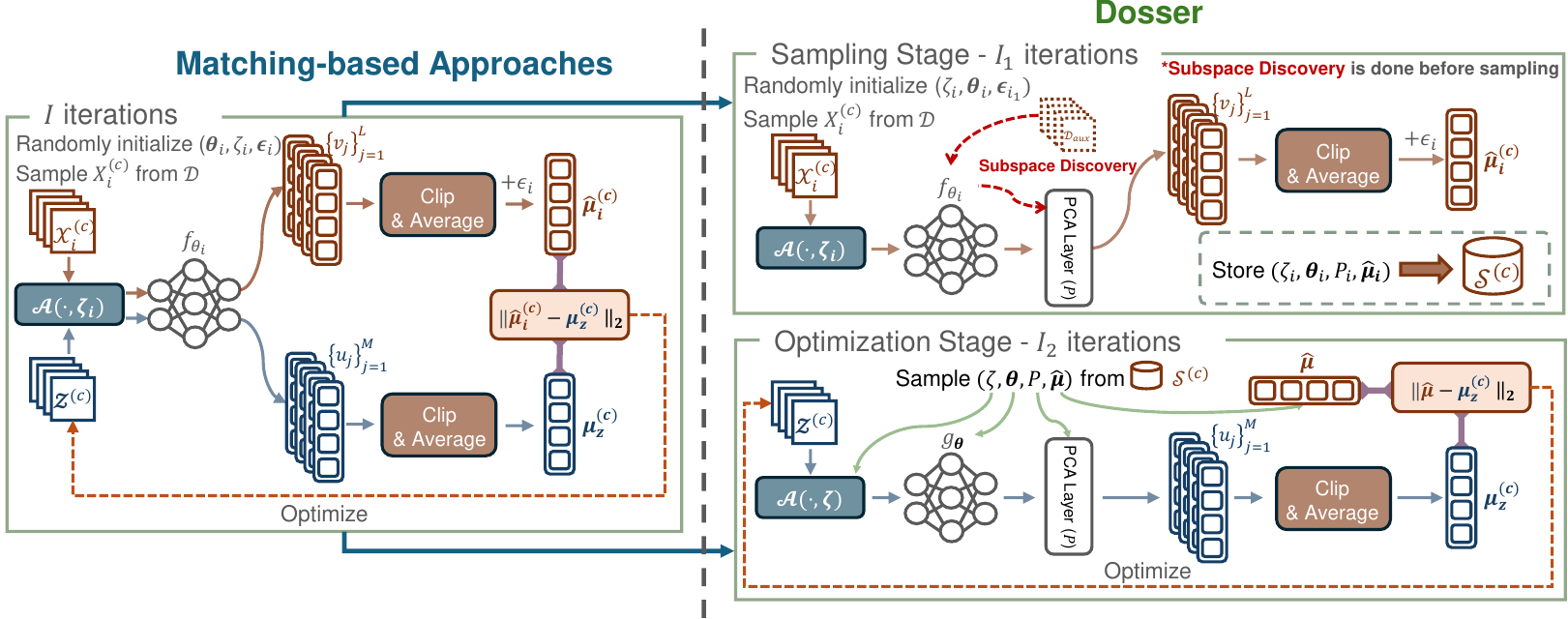}
    \caption{Overview of our proposed framework, which integrates Decoupled Optimization and Sampling (DOS) with Subspace Discovery for Error Reduction (SER).}
    \label{fig:method}
\end{figure*}

\subsection{Decoupled Optimization and Sampling (DOS)}
\label{sec:dos}


\subsubsection{Sampling Stage}
\label{sec:sampling}

In the sampling stage, for each class $c$, we perform the following steps to sample training signals at the $i^{th}$ iteration:

\begin{enumerate}[left=-12.5pt]
    \item \textbf{Data Sampling}: For each class $c$, sample a batch of images $\gX_i^{(c)} \sim \textsc{PoissonSample}\left( \gX^{(c)}, \frac{L}{N^{(c)}} \right)$ from the private dataset $\mathcal{D}^{(c)}$ using Poisson sampling with probability $L / N^{(c)}$, where $N^{(c)} = |\mathcal{D}^{(c)}|$ is the total number of images in the dataset, and $L$ represents the group size.

    \item \textbf{Signal Extraction}: 
    For each sampled $j^{th}$ image $\vx_{i,j}^{(c)}$ at the current $i^{th}$ iteration, apply a differentiable augmentation function $\mathcal{A}$ using a random seed $\zeta_i$, denoted as $\mathcal{A}_{\zeta_i}(\vx_{i,j}^{(c)})$. The augmentation function $\mathcal{A}$~\cite{zhao2021dataset} includes transformations such as random cropping, color saturations, and other techniques to enhance data diversity. 
    Next, extract training signals using a parameterized function $f_{\theta_i}$, where $f_{\theta_i}$ performs feature extraction in distribution-matching methods~\cite{zheng2022differentially} or gradient computation in gradient-matching methods~\cite{chen2022private}. The parameter set $\theta_i$ is reinitialized for each batch sampled, allowing $f_{\theta_i}$ to represent images across diverse signal spaces with randomly sampled parameters. 
    To satisfy differential privacy requirements, the extracted signals are clipped to limit their sensitivity by the clipping function $\text{clip}_K(\vv) = \vv \cdot \min\left(1, \frac{K}{\|\vv\|_2}\right).$ This ensures that the norm of $\vv$ does not exceed the threshold $K$. 
    The entire process of extracting the training signal $\vv_{i,j}$ from the sampled batch $\gX_i^{(c)}$ can be represented by the following function:
    \begin{equation}
    \vv_{i,j} = \mathcal{F}_{\zeta_i, \theta_i, K}(\vx_{i,j}^{(c)}) = \text{clip}_K \circ f_{\theta_i} \circ \mathcal{A}_{\zeta_i} \left(\vx_{i,j}^{(c)}\right).
    \label{eqn:F}
    \end{equation}
    \item \textbf{Aggregation and Noise Addition}: To ensure differential privacy, compute the aggregated signal and then add Gaussian noise:
    \begin{equation}
    \hat{\vmu}_i^{(c)} = \vmu_i^{(c)} +\mathbf{\veta}_i,
    \end{equation}
    where the aggregated signal $\vmu_i^{(c)}$ is defined as
    \begin{equation}
    \vmu_i^{(c)} = \frac{1}{L} \sum_{\vx_{i,j}^{(c)} \in \gX_i^{(c)}} \mathcal{F}_{\zeta_i, \theta_i, K}(\vx_{i,j}^{(c)}),
    \end{equation}
    and $\mathbf{\veta}_i \sim \mathcal{N}(0, \sigma^2 \mathbf{I})$ represents the Gaussian noise added for privacy. The noise scale $\sigma$ is determined based on the desired privacy budget $(\epsilon, \delta)$, and the calculation of $\sigma$ follows the process introduced by \citet{zheng2022differentially} and through the Opacus library~\cite{yousefpour2021opacus}.

\end{enumerate}


Following these steps, after $I_1$ iterations, we obtain a dataset of DP-protected training signals denoted by $\gS = \{\gS^{(c)}\}_{c=1}^{C}$ where $C$ is the number of classes. Each subset $\gS^{(c)} = \{(\hat{\boldsymbol{\mu}}_i, \zeta_i, \theta_i)\}_{i=1}^{I_1}$ contains $I_1$ tuples, with each tuple consisting of:

\begin{itemize}
    \item $\hat{\vmu}_i$: the noisy aggregated training signal,
    \item $\zeta_i$: the random seed used for data augmentation,
    \item $\theta_i$: the sampled model parameters at iteration $i$.
\end{itemize}

According to the post-processing theorem \cite{dp}, any operation on the DP-protected signal set that is independent of the private data does not incur additional privacy costs. Thus, we can repeatedly use the protected signal set for subsequent optimization steps.

\subsubsection{Optimization Stage}

In the optimization stage, we utilize the stored training signals to optimize a synthetic dataset $\gZ = \{\gZ^{(c)}\}_{c=1}^{C}$ for $I_2$ iterations, which is initialized with random Gaussian noise.
For each optimization iteration and each class $c$, the following steps are performed:

\begin{enumerate}[left=-12.5pt]
    \item \textbf{Signal Retrieval}: Randomly select $(\hat{\boldsymbol{\mu}}_i^{(c)}, \zeta_i, \theta_i) \in \gS^{(c)}$.

    \item \textbf{Synthetic Signal Computation}: Apply $\mathcal{F}$ in \cref{eqn:F} to the synthetic images with the stored random seed $\zeta_i$:
    \begin{equation}
    \vmu_{\mZ}^{(c)} = \frac{1}{M} \sum_{\mZ_j \in \gZ^{(c)}} \mathcal{F}_{\zeta_i, \theta_i, K}(\mZ_j).
    \end{equation}

    \item \textbf{Loss Calculation}: Compute the squared $\ell_2$ distance between the synthetic and noisy real aggregated signals:
    \begin{equation}
    \mathcal{L}^{(c)} = \left\| \hat{\vmu}_i^{(c)} - \boldsymbol{\mu}_{\mZ}^{(c)} \right\|_2^2.
    \end{equation}

    \item \textbf{Parameter Update}: Update the synthetic images $\gZ^{(c)}$ by performing gradient descent on the loss $\mathcal{L}^{(c)}$.
\end{enumerate}

By decoupling sampling and optimization, we can assign different numbers of iterations $I_1$ and $I_2$ to each process separately. This allows the optimization to converge better through longer iterations without introducing extra DP noise to the sampled signals.

\subsection{Subspace discovery for Error Reduction (SER)}


Improving the SNR in the raw extracted signal is another effective way to reduce the impact of later added DP noise. To improve the SNR, we introduce Subspace Discovery for Error Reduction (SER). SER leverages generative models to create auxiliary images that mimic the private dataset, enabling the identification of an informative subspace within a randomly initialized neural network. By projecting the signals onto the subspace, we effectively reduce the amount of noise captured by random neural networks, thereby enhancing the SNR and reducing the impact of the DP noise.

\paragraph{Theoretical Insights}
To understand the benefits of subspace projection in the context of differential privacy, we analyze the mean squared error (MSE) in estimating the true mean signal $\vmu$ with or without projection where the true mean of the signal $\vmu=\mathbb{E}_{\vx_j^{(c)}\in\gX^{(c)}}[\mathcal{F}(\vx_j^{(c)})]$. To perform this comparison, we start with the following basic assumption about the signal vector:

\begin{assumption}
    Each signal vector $\vv_j$, obtained by transforming a randomly sampled real data point $\vx_j$ using the function $\mathcal{F}$, can be modeled as:
    \begin{equation}
    \vv_j = \vmu + \vp_j + \vr_j, \quad \|\vv_j\|_2 \leq K,
    \end{equation}
    where $\vmu \in \mathbb{R}^D$ is the true mean signal vector, $\vp_i \in \mathbb{R}^D$ represents the \textbf{informative signal} with zero mean and covariance matrix $\Sigma_p$ (of rank $d$), and $\vr_i \in \mathbb{R}^D$ denotes the \textbf{uninformative signal} with zero mean and covariance $\Sigma_r$.
\end{assumption}

The differentially private noisy mean in the original space is calculated as
$
\hat{\vmu}_{\text{orig}} = \frac{1}{L} \sum_{\vv_i \in \gS} \vv_i + \mathbf{\mathbf{\veta}}_{\text{orig}},
$
where $\mathbf{\mathbf{\veta}}_{\text{orig}} \sim \mathcal{N}(0, \sigma_{\text{orig}}^2 \mathbb{I}_D)$ is Gaussian noise added to satisfy a given differential privacy budget $(\epsilon, \delta)$. Then, the noisy mean in the projected space is obtained by
$
\hat{\vmu}_{\text{proj}} = \frac{1}{L} \sum_{\vv_i \in \gS} \mP^\top \vv_i + \mathbf{\mathbf{\veta}}_{\text{proj}},
$
where $\mathbf{\mathbf{\veta}}_{\text{proj}} \sim \mathcal{N}(0, \sigma_{\text{proj}}^2 \mathbb{I}_d)$ is the Gaussian noise added in the projected space, with the same budget $(\epsilon, \delta)$. The noisy mean can then be reconstructed back to the original space via
$
\hat{\vmu}_{\text{back}} = \mP \hat{\vmu}_{\text{proj}}.
$
Under these conditions, the MSE in estimating the true mean $\vmu$ with and without projection can be defined as:

\begin{figure}
    \centering
    \includegraphics[width=\linewidth]{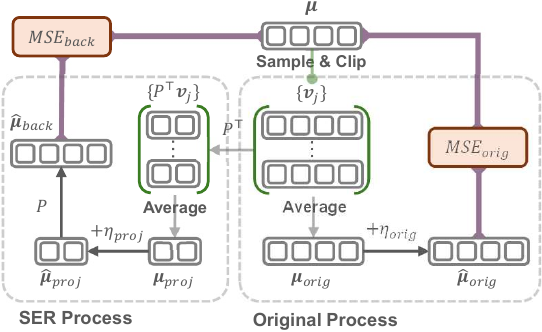}
    \caption{Illustration of Subspace discovery for Error Reduction (SER). We aim to achieve lower MSE within the differentially private framework by projecting training signals onto an informative subspace.}
    \label{fig:ser_mse}
\end{figure}

\begin{equation}
\left\{
\begin{array}{l}
\text{MSE}_{\text{orig}} = \mathbb{E}\left[ \left\| \hat{\vmu}_{\text{orig}} - \vmu \right\|_2^2 \right], \\ [5pt]
\text{MSE}_{\text{back}} = \mathbb{E}\left[ \left\| \hat{\vmu}_{\text{back}} - \vmu \right\|_2^2 \right].
\end{array}
\right.
\end{equation}

We formulate the difference between the two MSEs (as shown in \cref{fig:ser_mse}) with the following theorem:

\begin{reptheorem}{thm:ser}
    \label{thm:ser}
    Under the same budget of differential privacy $(\epsilon, \delta)$, the difference of MSE with and without projection $\mP$ in estimating the true mean $\vmu$ can be decomposed into the three terms:
    \begin{align*}
    &\text{MSE}_{\text{orig}} - \text{MSE}_{\text{back}} = 
    \underbrace{\frac{1}{L} \text{Tr}\left( (I - \mP \mP^\top) \Sigma_r \right)}_{\text{Projection Residual}} \\ 
    &+ \underbrace{\sigma_{\text{proj}}^2 \left( \frac{\max_j\|\vv_j\|_2^2}{\max_j\|\mP^\top \vv_j\|_2^2} D - d \right)}_{\text{Dimensional Reduction Effect}} \\
    &- \underbrace{\| (I - \mP \mP^\top) \vmu \|_2^2 + \frac{1}{L} \text{Tr}\left( (I - \mP \mP^\top) \Sigma_p \right)}_{\text{Projection Error}}.
    \end{align*}
\end{reptheorem}

Please refer to \cref{sec:theory} for the complete proof. We analyze the three terms separately as follows:





\paragraph{Projection Residual:} This term captures the variance in the uninformative signal excluded by the subspace $\mP$, quantifying the components discarded during projection.

\paragraph{Dimensional Reduction Effect:} This expression indicates that the reduction in MSE benefits from (1) the norm reduction after projection, given by $\frac{\|\vv_j\|_2^2}{\|\mP^\top \vv_j\|_2^2}$, and (2) the dimensionality reduction from $D$ to $d$.

\paragraph{Projection Error:} This term measures the error from projecting the true mean $\vmu$ and the informative signal variance $\Sigma_p$ into $\mP$. It depends on the subspace dimension $d$ and the alignment between the auxiliary and private datasets. A well-chosen subspace minimizes this error while preserving noise reduction benefits.

In summary, the theorem shows that dimensionality reduction minimizes error in estimating the true mean by (1) discarding uninformative variance in $\mP$ and (2) leveraging norm and dimensionality reduction post-projection. However, projection error introduces a trade-off, potentially impacting performance due to reduced dimensions and dataset discrepancies. To address this, we propose two methods for creating an auxiliary dataset $\gD_{\text{aux}}$:

\textbf{Leveraging Pre-trained Models:} We use a pre-trained foundation model such as Stable Diffusion (SD)~\cite{rombach2022high} to generate images for each category. This approach is particularly effective when the target dataset's distribution closely aligns with the pre-trained model's distribution. Since it does not involve the private dataset, it also incurs no additional privacy cost. One would argue that if we already have a generative model that can produce images for a specific class, a direct approach is to generate images for distillation. The key advantage of our method is that it ensures the distilled images we generate align well with the distribution of the target dataset, capturing its unique characteristics more accurately while using images from generative models could result in distribution discrepancy.

\textbf{Using Differentially Private Generative Models:} Given a total privacy budget $(\epsilon, \delta)$, we allocate a portion $(\epsilon_1, \delta_1)$ to train a generative model, such as a differentially private diffusion model (DPDM)~\cite{dockhorn2022differentially}, on $\gD$. We then perform SER using a generated dataset by the trained model and proceed with dataset distillation under the remaining privacy budget $(\epsilon_2, \delta_2)$, ensuring that $\epsilon_1 + \epsilon_2 = \epsilon$ and $\delta_1 + \delta_2 = \delta$. This method maintains the overall privacy budget, as formalized in the following theorem:

\begin{reptheorem}{thm:composition}
    \label{thm:composition}
    The process of distilling the private dataset $\mathcal{D}$ with an $(\epsilon_1, \delta_1)$-DP mechanism, supported by SER with an auxiliary dataset $\mathcal{D}_{\text{aux}}$ satisfying $(\epsilon_2, \delta_2)$-DP to $\mathcal{D}$, achieves $(\epsilon_1 + \epsilon_2, \delta_1 + \delta_2)$-DP to $\mathcal{D}$.
\end{reptheorem}

This theorem applies the basic composition theorem (see \cref{sec:theory} for the proof). Training a generative model on the private dataset requires an additional privacy cost. However, it is useful when the target domain is specialized, such as medical imaging or other niche fields not well-represented by foundational generative models typically trained on natural images.

\subsection{Overall Framework}
Our framework that combines DOS and SER for differentially private dataset distillation, named \textbf{Dosser}, is illustrated in \cref{fig:method}.
In the sampling stage for class $c$, for each iteration $i$, we initialize random neural networks and identify informative signal subspaces through PCA on the auxiliary data, giving the projection $\mP$. We then sample private data batch $\gX^{(c)}_i$ and obtain the training signal by $\hat{\mu}_i^{(c)} = \frac{1}{L}\sum_jP^\top \gF(\vx^{(c)}_{i,j})+\veta_i$. We store $\mP$ at each iteration $i$ along with the data tuples, forming sets $(\hat{\vmu}_i, \zeta_i, \theta_i, \mP_i)$ in $\gS$.
During the Optimization Stage, we iteratively update the synthetic dataset by aligning it with the stored noisy training signals within the identified subspaces by:
$$\mathcal{L}^{(c)} = \left\| \hat{\vmu}_i^{(c)} - \mP_i^\top\vmu_{\mZ}^{(c)} \right\|_2^2,$$
where $\vmu_{\mZ}^{(c)}=\frac{1}{M}\sum_{\vz^{(c)}\in\gZ^{(c)}}P^\top \gF(\vz^{(c)})$ is the averaged signal from the synthetic dataset in the subspace. 
This process leverages decoupled sampling to allow extensive optimization without additional privacy costs, while subspace discovery ensures that synthetic data captures the most relevant information from the original data.

\section{Experiments}

\subsection{Experimental Settings}
\paragraph{Dataset}
For empirical evaluation, we use the MNIST~\cite{deng2012mnist}, FashionMNIST~\cite{xiao2017/online}, and CIFAR-10~\cite{coates2011analysis} datasets. MNIST contains 70,000 $28 \times 28$ grayscale images of handwritten digits (0-9), with 60,000 for training and 10,000 for testing. FashionMNIST, a more challenging variant, includes 60,000 $28 \times 28$ grayscale images across 10 fashion categories, split into 50,000 training and 10,000 testing images. CIFAR-10 consists of 60,000 $32 \times 32$ color images across 10 classes, with 50,000 for training and 10,000 for testing.


\paragraph{Methods}
We evaluate the effectiveness of our method in comparison with several state-of-the-art differentially private data distillation approaches under a strict privacy budget of $(\varepsilon = 1, \delta = 10^{-5})$. The methods we compare include DP-Sinkhorn~\cite{cao2021don}, DP-MERF~\cite{harder2021dp}, PSG~\cite{chen2022private}, DP-KIP-ScatterNet~\cite{vinaroz2024differentially}, and NDPDC~\cite{zheng2022differentially}. As a baseline, we also compare the above-mentioned methods with standard distribution matching without differential privacy, noted as DM w/o DP, to better understand the performance gap with and without differential privacy. We implement Dosser with distribution matching, based on the framework established by \citet{zheng2022differentially}. In our method, unless otherwise specified, we set the sampling iteration to 10,000, the optimization iteration to 200,000, and the privacy budget to $(1, 10^{-5})$. We also adopted Partitioning and Expansion Augmentation (PEA) from Improved Distribution Matching~\cite{zhao2023improved}, which is a technique to enhance dataset distillation by splitting and enlarging synthetic images. For SER, on MNIST and FashionMNIST, we construct an auxiliary dataset by training a differentially private diffusion model (DPDM)~\cite{dockhorn2022differentially} on the private dataset with a privacy budget of $(0.2\epsilon, 0.2\delta)$, followed by dataset distillation with $(0.8\epsilon, 0.8\delta)$, ensuring an overall privacy budget of $(\epsilon, \delta)$ as outlined in \Cref{thm:composition}. We determine the privacy‑budget allocation empirically, selecting the split that yields the highest validation accuracy. The results of using other DP generators can be found in \Cref{sec:other_models}. For CIFAR-10, we construct the auxiliary dataset directly using SD-v1-4~\cite{rombach2022high}. We set the subspace dimension to 500 and the auxiliary dataset size to 1000; additional details are in \Cref{sec:settings}.


\subsection{Evaluation Against Baselines}

We compare the accuracies of various methods on MNIST, FashionMNIST, and CIFAR-10 under $(1, 10^{-5})$-DP with IPC of 10 and 50. The results are shown in \Cref{tab:main}. 
DM without differential privacy, highlighted in green, achieves the highest accuracy across datasets, showing the upper-performance limit without the added privacy constraints. Matching-based methods are highlighted in blue rows. Among them, NDPDC, which is derived from DM, exhibits noticeable accuracy degradation due to the addition of DP noise. This comparison directly highlights the impact of privacy noise on model performance.
Our method, Dosser, builds upon NDPDC by enhancing the matching process with DOS and SER, improving its accuracy under the same privacy constraints. These additions increase the utility of the training signal, allowing Dosser to achieve higher accuracy than NDPDC. Specifically, Dosser provides an average improvement of $1.6\%$ on MNIST, $2.1\%$ on FashionMNIST, and $10.6\%$ on CIFAR-10, with more substantial gains observed on the more complex datasets. Notably, Dosser exhibits a much smaller accuracy gap to the original DM without differential privacy. This difference is especially apparent on CIFAR-10 with IPC=10, where Dosser’s performance gap from DM w/o DP is only $1.5\%$, compared to NDPDC’s $11.7\%$ gap. In general, Dosser's strong performance in datasets, especially with close accuracy with DM w/o DP, demonstrates its superior ability to mitigate the effects of noise within the differential privacy framework.

\begin{table*}[t]
    \centering
    \scriptsize
    \resizebox{\linewidth}{!}
    {
    \begin{tabular}{l|cc|cc|cc}
        \toprule
         & \multicolumn{2}{c}{\textbf{MNIST}} & \multicolumn{2}{c}{\textbf{FashionMNIST}} & \multicolumn{2}{c}{\textbf{CIFAR-10}}\\
         \midrule
         Method & IPC=$10$ & IPC=$50$ & IPC=$10$ & IPC=$50$ & IPC=$10$ & IPC=$50$ \\
         \midrule
         \rowcolor{woDP}
         DM w/o DP & $97.8$ & $99.2$ & $84.6$ & $88.7$ & $52.1$ & $60.6$ \\
         DP-Sinkhorn~\cite{cao2021don} & $31.7\pm3.2$ & $33.9\pm1.7$ & $9.8\pm0.0$ & $22.0\pm0.1$ & $ - $ & $ - $ \\
         DP-MERF~\cite{harder2021dp} & $75.0\pm0.3$ & $84.4\pm2.3$ & $65.5\pm3.2$ & $71.3\pm1.7$ & $ - $ & $ - $ \\
        
         DP-KIP-ScatterNet~\cite{vinaroz2024differentially}& $25.8\pm2.1$ & $13.8\pm2.6$ & $17.7\pm1.5$ & $16.2\pm1.2$ & $16.8\pm1.1$ & $9.5\pm0.5$ \\
         
         \rowcolor{matching}
         PSG~\cite{chen2022private} & $78.6\pm0.7$& $-$ & $68.5\pm0.5$& $-$& $33.6\pm0.3$ & $-$\\
         
         \rowcolor{matching}
         NDPDC~\cite{zheng2022differentially}& $93.1\pm0.4$ & $94.1\pm0.4$ & $77.7\pm0.6$ & $78.8\pm0.4$ & $39.4\pm0.8$ & $42.3\pm0.8$ \\
         
         \rowcolor{matching}
         \textbf{Dosser (ours)}& $ 95.3\pm0.0 $ & $96.4\pm0.0$ & $\mathbf{81.6}\pm0.1$ & $81.8\pm0.2$ & $44.2\pm0.2$ & $49.1\pm0.5$ \\
         
         \rowcolor{matching}
         \textbf{Dosser (ours)  w/ PEA~\cite{zhao2023improved}}& $ \mathbf{96.4}\pm0.0 $ & $\mathbf{96.7}\pm0.1$ & $80.1\pm0.5$ & $\mathbf{83.1}\pm0.5$ & $\mathbf{50.6}\pm0.1$ & $\mathbf{52.3}\pm0.6$ \\
         
         \bottomrule
    \end{tabular}
    }
    \caption{The table presents a comparison of accuracies achieved by various methods on three datasets: MNIST, FashionMNIST, and CIFAR-10, evaluated under a privacy budget of $(1, 10^{-5})$. Each method’s performance is reported for IPC of 10 and 50.}
    \label{tab:main}
\end{table*}

\subsection{Ablation Studies}
In this section, we conduct ablation studies on three key aspects: 1) evaluating the performance gain contributed by each of the proposed modules, 2) examining the effect of increasing the number of training iterations to show how DOS improves performance through additional optimization steps, and 3) analyzing the impact of varying the dimensionality of the projected subspace in SER, as well as the amount of auxiliary data used in SER.

\subsubsection{Ablating Contributions of DOS and SER}

In this study, we evaluate the individual contributions of DOS and SER, the results on CIFAR-10 are shown in \Cref{tab:ablation}. When applying DOS alone, accuracy improves by an average of $4.3\%$ compared to the baseline without DOS and SER, demonstrating the advantage of additional optimization steps. Applying SER alone also enhances results due to the increased signal-to-noise ratio introduced by SER; however, the improvement is more modest at around $1.9\%$. The improvement is limited when the noise required is small ($\epsilon=10$). When both DOS and SER are applied together, the combined benefits are clear: SER effectively capitalizes on the additional optimization steps provided by DOS, resulting in a $5.7\%$ improvement over the baseline. This indicates that while SER enhances the signal-to-noise ratio, additional training iterations are essential for achieving optimal convergence. DOS and SER complement each other by enhancing the utility of training signals from two different aspects, and achieve better performance when combined.

\begin{table}[t]
    \centering
    \resizebox{\linewidth}{!}
    {
    \begin{tabular}{cc|cc|cc}
        \toprule
         \multicolumn{2}{c|}{Dosser Components} & \multicolumn{4}{c}{(IPC, $\epsilon$)} \\
         \midrule
         DOS & SER & $(10, 1)$ & $(50, 1)$ & $(10, 10)$ & $(50, 10)$ \\
         \midrule
         
         \xmark & \xmark & $41.7\pm0.0$ & $45.7\pm0.1$ & $54.1\pm0.0$ & $57.7\pm0.0$ \\
         
         \cmark & \xmark & $47.7\pm0.1$ & $51.0\pm0.2$ & $56.7\pm0.5$ & $\mathbf{61.1}\pm0.1$ \\
         
         \xmark & \cmark & $46.5\pm0.2$ & $47.8\pm0.3$ & $54.7\pm0.5$ & $57.6\pm0.0$ \\
         
         \cmark & \cmark & $\mathbf{50.6}\pm0.1$ & $\mathbf{52.3}\pm0.3$ & $\mathbf{58.0}\pm0.2$ & $61.0\pm0.0$ \\

         \bottomrule
    \end{tabular}
    }
    \caption{Performance improvements from individual and combined contributions of DOS and SER components under varying IPC and privacy settings on CIFAR-10.}
    \label{tab:ablation}
\end{table}

\begin{figure}[t!]
    \centering
    \begin{subfigure}{0.5\linewidth}
        \centering
        \includegraphics[width=\linewidth]{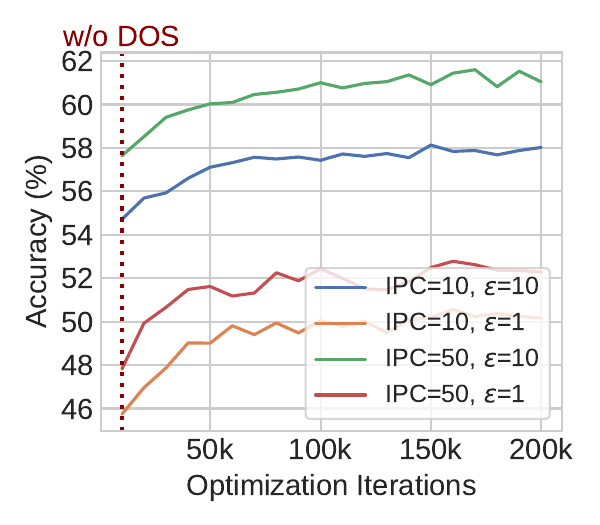}
        \caption{Varying optimization iterations \\ in DOS.}
        \label{fig:dos}
    \end{subfigure}%
    \hfill
    \begin{subfigure}{0.5\linewidth}
        \centering
        \includegraphics[width=\linewidth]{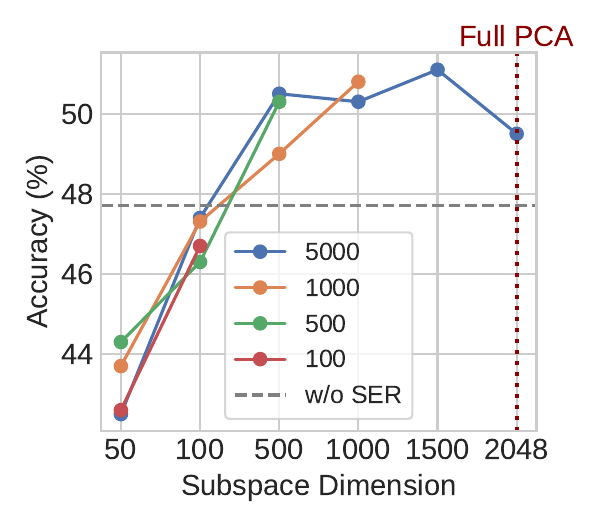}
        \caption{Varying subspace dimensions and auxiliary dataset sizes in SER.}
        \label{fig:ser}
    \end{subfigure}
    \caption{Ablation studies on CIFAR-10 with ConvNet.}
    \label{fig:combined_figure}
\end{figure}

\subsubsection{Impact of DOS Hyperparameters}
To investigate the impact of hyperparameters in DOS, specifically the number of optimization iterations, we analyze the accuracy changes during optimization in four settings varying IPC and $\epsilon$, while keeping $\delta=10^{-5}$. The results on CIFAR-10 are presented in \Cref{fig:dos}. 
In each setting, the sampling iteration count is set to 10,000, so the leftmost point (where the optimization iteration equals 10,000) can be regarded as the scenario without DOS. From the figure, we observe that without DOS, the optimization process of synthetic images does not fully converge, resulting in low test accuracy for downstream tasks. As the optimization iterations increase, accuracy gradually improves and ultimately reaches a much higher level than without DOS, demonstrating the necessity of decoupling sampling and optimization.
The effect of DOS is particularly pronounced with higher IPC values; the accuracy gap between the initial and final evaluations is greater when IPC$=50$ than when IPC$=10$, regardless of whether $\epsilon=1$ or $\epsilon=10$.

\subsubsection{Impact of SER Hyperparameters}
To examine how subspace dimensionality and the size of the auxiliary dataset in SER influence matching efficiency, we analyze downstream accuracy on CIFAR-10 across various settings, as illustrated in \Cref{fig:ser}. The ConvNet feature dimension is 2048. Our results indicate minimal impact from varying the auxiliary dataset size, likely because the variance in PCA parameters across different auxiliary dataset sizes is small and thus has little effect on matching performance. The only benefit of increasing dataset size is to increase the maximum dimension we can project with PCA.
A notable observation is the substantial accuracy gap between using the full 2048-dimensional feature space and the results without SER. When the reduced dimensionality is set to 2048, PCA effectively performs as a single linear transformation that concentrates high-variance components in the top dimensions. This suggests that PCA’s benefits are not solely due to error reduction via dimensionality reduction but also from emphasizing high-variance components. Since the Gaussian noise is uniformly distributed across all dimensions, concentrating high-variance signals into fewer components enhances the signal-to-noise ratio in those dimensions, thereby improving matching efficiency.
We conducted additional quantitative experiments to assess the direct impact of our method to the MSE estimation via noise reduction; see \Cref{sec:more_quantitative_analysis} for details.
\section{Limitations}
A limitation of our method is that it is specifically designed for matching training signals from randomly initialized networks, which is less competitive. Recent advances in DD, especially those scaling to larger datasets like ImageNet~\cite{krizhevsky2012imagenet}, often require pre-trained models for extracting matching signal or based on trajectory matching etc., which have not yet been adapted to DP-constrained scenarios. In future work, we aim to explore ways to adapt our approach to more advanced matching-based DD techniques or other state-of-the-art DD methods. Another limitation lies in SER, which requires an auxiliary dataset that closely matches the distribution of the training data. For natural image datasets or large datasets, we can create this auxiliary dataset using foundational generative models or by training a generative model with DP on the large dataset. However, for specialized domain datasets with limited data, the performance of SER may be constrained.

\section{Conclusion}

In this paper, we introduced a novel framework for differentially private dataset distillation that combines two key innovations: decoupling the sampling and optimization processes and applying subspace projection to improve signal utility. Our approach addresses the limitations of existing matching-based methods by enabling independent control over sampling and optimization iterations, which reduces cumulative noise injection and allows for more efficient utilization of the privacy budget. Additionally, our use of subspace projection identifies and focuses on the most informative signal subspace, effectively increasing the signal-to-noise ratio within each training signal.
Experimental results across multiple datasets validate that our framework achieves superior accuracy and privacy efficiency compared to traditional methods. Our method offers a substantial improvement in differentially private dataset distillation, setting a new standard for privacy-preserving data synthesis.


\section*{Acknowledgements}
We would like to thank Nicholas Apostoloff, Oncel Tuzel, and Jerremy Holland for their insightful comments and constructive feedback throughout the development of this work. Their suggestions helped clarify key ideas and strengthen both the methodology and presentation. 
\newpage

{
    \small
    \bibliographystyle{ieeenat_fullname}
    \bibliography{main}
}

\clearpage
\setcounter{page}{1}
\maketitlesupplementary

\renewcommand{\thesection}{\Alph{section}}  
\renewcommand{\theHsection}{Appendix.\thesection}  
\makeatletter
\renewcommand{\@seccntformat}[1]{~\csname the#1\endcsname\quad}  
\makeatother
\appendix
\crefname{section}{Appendix}{Appendices}  

\section{Theoretical Analysis}
\label{sec:theory}

\begin{lemma}
\label{lemma:ratio}
Let $\gV = \{\vv_j\}_{j=1}^{L}$ be a dataset of $L$ samples $\vv_i \in \mathbb{R}^D$. Let $\mP \in \mathbb{R}^{D \times d}$ be a matrix with orthonormal rows (i.e., $\mP \mP^\top = \mathbb{I}_d$). Suppose that adding Gaussian noise $\veta \sim \mathcal{N}(0, \sigma_1^2 \mathbb{I}_D)$ to the sample mean ensures $(\epsilon, \delta)$-differential privacy:
$$
\hat{\mu}_{\text{orig}} = \frac{1}{L} \sum_{j=1}^{L} \vv_j + \veta.
$$
Then, adding Gaussian noise $\veta' \sim \mathcal{N}(0, \sigma_2^2 \mathbb{I}_d)$ to the projected sample mean:
$$
\hat{\vv}_{\text{back}} = \mP\left( \frac{1}{L} \sum_{j=1}^{L} \mP^\top \vv_j + \veta' \right),
$$
ensures $(\epsilon, \delta)$-differential privacy, provided that
$$
\frac{\sigma_1}{\sigma_2} = \frac{\max_j \|\vv_j\|_2}{\max_j \|\mP^\top \vv_j\|_2}.
$$
\end{lemma}
\begin{proof}
We start by recalling that the Gaussian mechanism provides $(\epsilon, \delta)$-differential privacy when noise drawn from $\mathcal{N}(0, \sigma^2 \mathbb{I})$ is added to a function $\mathcal{M}$, where the noise scale $\sigma$ is proportional to the function's $\ell_2$-sensitivity $\Delta_\mathcal{M}$.

The sensitivity of the sample mean function $\mathcal{M}_{\text{orig}}(\gV) = \frac{1}{L} \sum_{i=1}^{L} \vv_i$ is given by
$$
\Delta_{\text{orig}} = \max_{\gV,\, \gV'} \left\| \mathcal{M}_{\text{orig}}(\gV) - \mathcal{M}_{\text{orig}}(\gV') \right\|_2,
$$
where $\gV$ and $\gV'$ differ in at most one element. The maximum change occurs when one sample is replaced, yielding
$$
\Delta_{\text{orig}} = \frac{1}{L} \max_i \| \vv_i \|_2.
$$

Similarly, for the projected mean function $\mathcal{M}_{\text{proj}}(\gV) = \frac{1}{L} \sum_{i=1}^{L} \mP^\top \vv_i$, the sensitivity is
$$
\Delta_{\text{proj}} = \frac{1}{L} \max_i \left\| \mP^\top \vv_i \right\|_2.
$$

The Gaussian mechanism requires the noise scale $\sigma$ to be proportional to the sensitivity. Therefore, the ratio of the noise scales should match the ratio of sensitivities:
$$
\frac{\sigma_1}{\sigma_2} = \frac{\Delta_{\text{orig}}}{\Delta_{\text{proj}}} = \frac{\max_i \| \vv_i \|_2}{\max_i \left\| \mP^\top \vv_i \right\|_2}.
$$

\end{proof}

\begin{theorem}
\label{thm:ser}
    Under the same budget of differential privacy $(\epsilon, \delta$, the difference of MSE with and without projection $\mP$ in estimating the true mean $\vmu$ can be decomposed into the three terms:
    \begin{align*}
        &\text{MSE}_{\text{orig}} - \text{MSE}_{\text{back}} = 
        \underbrace{\frac{1}{L} \text{Tr}\left( (I - \mP \mP^\top) \Sigma_r \right)}_{\text{Projection Residual}} \\ 
        &+ \underbrace{\sigma_{\text{proj}}^2 \left( \frac{\|\vv_j\|_2^2}{\|\mP^\top \vv_j\|_2^2} D - d \right)}_{\text{Dimensional Reduction Effect}} \\
        &- \underbrace{\| (I - \mP \mP^\top) \vmu \|_2^2 + \frac{1}{L} \text{Tr}\left( (I - \mP \mP^\top) \Sigma_p \right)}_{\text{Projection Error}}.
    \end{align*}
\end{theorem}

\begin{proof}
We analyze the MSE in both the original and projected spaces to establish the theorem.

First, consider the noisy mean in the original space:
$$
\hat{\vmu}_{\text{orig}} = \vmu + \mathbf{\vomega} + \mathbf{\veta}_{\text{orig}},
$$
where $\mathbf{\vomega} = \frac{1}{L} \sum_{j} (\vp_j + \vr_j)$ represents the sampling deviation from the true mean due to finite sample size and inherent data variability.

The MSE in the original space is then:
$$
\text{MSE}_{\text{orig}} = \mathbb{E}\left[ \left\| \hat{\vmu}_{\text{orig}} - \vmu \right\|_2^2 \right] = \mathbb{E}\left[ \left\| \mathbf{\vomega} + \mathbf{\veta}_{\text{orig}} \right\|_2^2 \right].
$$
Expanding the squared norm, we obtain:
$$
\text{MSE}_{\text{orig}} = \mathbb{E}\left[ \| \mathbf{\vomega} \|_2^2 \right] + \mathbb{E}\left[ \| \mathbf{\veta}_{\text{orig}} \|_2^2 \right] + 2 \mathbb{E}\left[ \mathbf{\vomega}^\top \mathbf{\veta}_{\text{orig}} \right].
$$
Since $\mathbf{\vomega}$ and $\mathbf{\veta}_{\text{orig}}$ are independent and both have zero mean, the cross term vanishes:
$$
\mathbb{E}\left[ \mathbf{\vomega}^\top \mathbf{\veta}_{\text{orig}} \right] = 0.
$$
Thus, the MSE in the original space simplifies to:
$$
\text{MSE}_{\text{orig}} = \mathbb{E}\left[ \| \mathbf{\vomega} \|_2^2 \right] + \mathbb{E}\left[ \| \mathbf{\veta}_{\text{orig}} \|_2^2 \right].
$$

Next, consider the noisy mean in the projected space:
$$
\hat{\vmu}_{\text{proj}} = \mP^\top (\vmu + \mathbf{\vomega}) + \mathbf{\veta}_{\text{proj}},
$$
and the reconstructed noisy mean in the original space:
$$
\hat{\vmu}_{\text{back}} = \mP \hat{\vmu}_{\text{proj}} = \mP \mP^\top (\vmu + \mathbf{\vomega}) + \mP \mathbf{\veta}_{\text{proj}}.
$$

We introduce an error term to account for the recover error from PCA transformation. Specifically, define:
$$
\xi_\mP = \mP \mP^\top \vmu - \vmu,
$$
which quantifies the deviation of the true mean $\vmu$ from its projection onto the subspace spanned by $\mP$. If $\mP$ perfectly captures the mean, then $\xi_\mP = 0$. Otherwise, $\xi_\mP$ represents the component of $\vmu$ orthogonal to the subspace spanned by $\mP$.

Substituting this into the expression for $\hat{\vmu}_{\text{back}}$, we obtain:
$$
\hat{\vmu}_{\text{back}} = \vmu + \mP \mP^\top \mathbf{\vomega} + \mP \mathbf{\veta}_{\text{proj}} + \xi_\mP.
$$

The MSE in the projected and reconstructed space is therefore:
\begin{align*}
\text{MSE}_{\text{back}} &= \mathbb{E}\left[ \left\| \hat{\vmu}_{\text{back}} - \vmu \right\|_2^2 \right] \\
&= \mathbb{E}\left[ \left\| \mP \mP^\top \mathbf{\vomega} + \mP \mathbf{\veta}_{\text{proj}} + \xi_\mP \right\|_2^2 \right] \\
&= \mathbb{E}\left[ \| \mP \mP^\top \mathbf{\vomega} \|_2^2 \right] + \mathbb{E}\left[ \| \mP \mathbf{\veta}_{\text{proj}} \|_2^2 \right] + \mathbb{E}\left[ \| \xi_\mP \|_2^2 \right]\\ & +  2 \mathbb{E}\left[ (\mP \mP^\top \mathbf{\vomega})^\top (\mP \mathbf{\veta}_{\text{proj}}) \right] + 2 \mathbb{E}\left[ (\mP \mP^\top \mathbf{\vomega})^\top \xi_\mP \right] \\ &+ 2 \mathbb{E}\left[ (\mP \mathbf{\veta}_{\text{proj}})^\top \xi_\mP \right].
\end{align*}
Given that $\mathbf{\vomega}$, $\mathbf{\veta}_{\text{proj}}$, and $\xi_\mP$ are all zero-mean and mutually independent, the cross terms vanish:
$$
\left\{
\begin{array}{l}
\mathbb{E}\left[ (\mP \mP^\top \mathbf{\vomega})^\top (\mP \mathbf{\veta}_{\text{proj}}) \right] = 0, \\ [10pt]
\mathbb{E}\left[ (\mP \mP^\top \mathbf{\vomega})^\top \xi_\mP \right] = 0, \\ [10pt]
\mathbb{E}\left[ (\mP \mathbf{\veta}_{\text{proj}})^\top \xi_\mP \right] = 0.
\end{array}
\right.
$$

Thus, the MSE in the projected and reconstructed space simplifies to:
$$
\text{MSE}_{\text{back}} = \mathbb{E}\left[ \| \mP \mP^\top \mathbf{\vomega} \|_2^2 \right] + \mathbb{E}\left[ \| \mP \mathbf{\veta}_{\text{proj}} \|_2^2 \right] + \mathbb{E}\left[ \| \xi_\mP \|_2^2 \right].
$$

To evaluate these expectations, we consider the properties of covariance matrices. The covariance of $\mathbf{\vomega}$ is:
$$
\text{Cov}\left( \mathbf{\vomega} \right) = \frac{1}{L} \left( \Sigma_p + \Sigma_r \right).
$$
Thus, the first term becomes:
\begin{align*}
\mathbb{E}\left[ \| \mP \mP^\top \mathbf{\vomega} \|_2^2 \right] &= \text{Tr}\left( \mP \mP^\top \text{Cov}\left( \mathbf{\vomega} \right) \right) \\
&= \frac{1}{L} \text{Tr}\left( \mP \mP^\top (\Sigma_p + \Sigma_r) \right).
\end{align*}

For the second term, since $\mathbf{\veta}_{\text{proj}} \sim \mathcal{N}(0, \sigma_{\text{proj}}^2 I_d)$, we have:
\begin{align*}
\mathbb{E}\left[ \| \mP \mathbf{\veta}_{\text{proj}} \|_2^2 \right] &= \text{Tr}\left( \mP^\top \mP \mathbb{E}\left[ \mathbf{\veta}_{\text{proj}} \mathbf{\veta}_{\text{proj}}^\top \right] \right) \\ &= \text{Tr}\left( \mP^\top \mP \sigma_{\text{proj}}^2 I_d \right) \\ &= \sigma_{\text{proj}}^2 \text{Tr}\left( \mP^\top \mP \right) \\
&= \sigma_{\text{proj}}^2 d.
\end{align*}
The third term, $\mathbb{E}\left[ \| \xi_\mP \|_2^2 \right]$, quantifies the error between the mean estimated in the subspace and its projection back to the original space compared to the true mean:
$$
\mathbb{E}\left[ \| \xi_\mP \|_2^2 \right] = \| \xi_\mP \|_2^2 = \| \mP \mP^\top \vmu - \vmu \|_2^2.
$$

Therefore, the MSE in the projected and reconstructed space is:
$$
\text{MSE}_{\text{back}} = \frac{1}{L} \text{Tr}\left( \mP \mP^\top (\Sigma_p + \Sigma_r) \right) + \sigma_{\text{proj}}^2 d + \| \mP \mP^\top \vmu - \vmu \|_2^2.
$$

Comparing this with the MSE in the original space:
$$
\text{MSE}_{\text{orig}} = \frac{1}{L} \text{Tr}\left( \Sigma_p + \Sigma_r \right) + \sigma_{\text{orig}}^2 D,
$$
we define the difference $\Delta$ as:

\begin{align*}
\Delta &= \text{MSE}_{\text{orig}} - \text{MSE}_{\text{back}} \\ &= \frac{1}{L} \text{Tr}\left( \Sigma_p + \Sigma_r \right) + \sigma_{\text{orig}}^2 D \\
&- \left( \frac{1}{L} \text{Tr}\left( \mP \mP^\top (\Sigma_p + \Sigma_r) \right) + \sigma_{\text{proj}}^2 d + \| \mP \mP^\top \vmu - \vmu \|_2^2 \right).
\end{align*}
Simplifying the trace terms, we observe that:
\begin{align*}
\text{Tr}\left( \Sigma_p + \Sigma_r \right) &- \text{Tr}\left( \mP \mP^\top (\Sigma_p + \Sigma_r) \right) \\& = \text{Tr}\left( (I - \mP \mP^\top) (\Sigma_p + \Sigma_r) \right).
\end{align*}
According to \cref{lemma:ratio}, we have:
$$
\sigma_{\text{orig}}^2 D - \sigma_{\text{proj}}^2 d = \sigma_{\text{proj}}^2 \left( \frac{\max_j\|\vv_j\|_2^2}{\max_j\|\mP^\top \vv_j\|_2^2} D - d \right).
$$
Substituting above into the expression for $\mathbf{\vomega}$, we obtain:
    \begin{align*}
        &\Delta = 
        \underbrace{\frac{1}{L} \text{Tr}\left( (I - \mP \mP^\top) \Sigma_r \right)}_{\text{Projection Residual}} \\ 
        &+ \underbrace{\sigma_{\text{proj}}^2 \left( \frac{\max_j\|\vv_j\|_2^2}{\max_j\|\mP^\top \vv_j\|_2^2} D - d \right)}_{\text{Dimensional Reduction Effect}} \\
        &- \underbrace{\| (I - \mP \mP^\top) \vmu \|_2^2 + \frac{1}{L} \text{Tr}\left( (I - \mP \mP^\top) \Sigma_p \right)}_{\text{Projection Error}}.
    \end{align*}
\end{proof}

\begin{theorem}
    \label{thm:composition}
    The process of distilling the private dataset $\mathcal{D}$ with an $(\epsilon_1, \delta_1)$-DP mechanism, supported by SER with an auxiliary dataset $\mathcal{D}_{\text{aux}}$ satisfying $(\epsilon_2, \delta2)$-DP to $\mathcal{D}$, achieves $(\epsilon_1 + \epsilon_2, \delta_1 + \delta_2)$-DP to $\mathcal{D}$.
\end{theorem}

\begin{proof}
    
To prove Theorem \ref{thm:composition}, we utilize fundamental properties of differential privacy, specifically the \textit{Basic Composition Theorem} and the \textit{Post-Processing Theorem}.

\begin{lemma}[Basic Composition Theorem \cite{dp}]
\label{lemma:basic}
    If a randomized mechanism $\mathcal{M}_1$ satisfies $(\epsilon_1, \delta_1)$-DP and another randomized mechanism $\mathcal{M}_2$ satisfies $(\epsilon_2, \delta_2)$-DP, then the sequential composition of these mechanisms, defined as $\mathcal{M} = \mathcal{M}_2 \circ \mathcal{M}_1$, satisfies $(\epsilon_1 + \epsilon_2, \delta_1 + \delta_2)$-DP.
\end{lemma}

\begin{lemma}[Post-Processing Theorem \cite{dp}]
\label{lemma:post}
Any data-independent transformation of the output of a differentially private mechanism does not degrade its privacy guarantees. Formally, if $\mathcal{M}$ satisfies $(\epsilon, \delta)$-DP, then for any deterministic or randomized function $f$, the mechanism $f \circ \mathcal{M}$ also satisfies $(\epsilon, \delta)$-DP.
\end{lemma}

We define the two mechanisms involved in the process as follows.

Let $\mathcal{M}_1$ represent the mechanism responsible for SER. The input to $\mathcal{M}_1$ is the private dataset $\mathcal{D}$, and its output is the auxiliary dataset $\mathcal{D}_{\text{aux}}$. By assumption, $\mathcal{M}_1$ satisfies $(\epsilon_1, \delta_1)$-differential privacy with respect to $\mathcal{D}$.

Let $\mathcal{M}_2$ represent the mechanism responsible for the distillation process. The inputs to $\mathcal{M}_2$ are the private dataset $\mathcal{D}$ and the auxiliary dataset $\mathcal{D}_{\text{aux}}$, and its output is the distilled dataset $\mathcal{Z}$. By assumption, $\mathcal{M}_2$ satisfies $(\epsilon_2, \delta_2)$-differential privacy with respect to $\mathcal{D}$.

It is important to note that $\mathcal{M}_2$ utilizes $\mathcal{D}_{\text{aux}}$, which is already the output of $\mathcal{M}_1$. However, since $\mathcal{M}_1$ ensures that $\mathcal{D}_{\text{aux}}$ is $(\epsilon_1, \delta_1)$-DP with respect to $\mathcal{D}$, any further processing of $\mathcal{D}_{\text{aux}}$ by $\mathcal{M}_2$ is considered post-processing of a DP-protected output.

Applying \cref{lemma:post}, the usage of $\mathcal{D}_{\text{aux}}$ by $\mathcal{M}_2$ does not introduce any additional privacy loss beyond what is already accounted for by $\mathcal{M}_1$. Therefore, $\mathcal{M}_2$ maintains its $(\epsilon_2, \delta_2)$-DP guarantee with respect to $\mathcal{D}$ independently of $\mathcal{D}_{\text{aux}}$.

Since $\mathcal{M}_1$ and $\mathcal{M}_2$ are applied sequentially, we apply \cref{lemma:post}. The cumulative privacy loss incurred by applying both mechanisms in sequence is the sum of their individual privacy parameters.

Formally, the overall mechanism $\mathcal{M}$, defined as:
$$
\mathcal{M} = \mathcal{M}_2 \circ \mathcal{M}_1
$$
satisfies:
$$
\mathcal{M} \text{ satisfies } (\epsilon_1 + \epsilon_2, \delta_1 + \delta_2)\text{-DP}.
$$

By sequentially applying $\mathcal{M}_1$ and $\mathcal{M}_2$, and leveraging both the Basic Composition and Post-Processing Theorems, we conclude that the combined process satisfies $(\epsilon_1 + \epsilon_2, \delta_1 + \delta_2)$-DP with respect to the private dataset $\mathcal{D}$.

\end{proof}
\section{Additional Quantitative Analysis}
\label{sec:more_quantitative_analysis}

\subsection{Effect of Privacy‑Budget Split}
\label{app:budget-split}

Table~\ref{tab:privacy_splits} shows how allocating the total budget $(\epsilon{=}1.0)$ between auxiliary data generation ($\epsilon_1$) and DP-based optimation ($\epsilon_2$) affects downstream accuracy. We observe that allocating \,$\epsilon_1{:}\epsilon_2{=}\,0.8{:}0.2$\, offers a good trade-off.

\begin{table}[h]
    \centering
    \resizebox{\linewidth}{!}{
    \begin{tabular}{c|cc|cc|cc|cc|cc}
    \toprule
    \textbf{($\epsilon_1$, $\epsilon_2$)} & \multicolumn{2}{c|}{\textbf{(0.9, 0.1)}} & \multicolumn{2}{c|}{\textbf{(0.8, 0.2)}} & \multicolumn{2}{c|}{\textbf{(0.7, 0.3)}} & \multicolumn{2}{c|}{\textbf{(0.6, 0.4)}} & \multicolumn{2}{c}{\textbf{(0.5, 0.5)}} \\
    \midrule
    \textbf{IPC} & 10 & 50 & 10 & 50 & 10 & 50 & 10 & 50 & 10 & 50 \\
    \midrule
    MNIST & 96.3 & 96.5 & 96.4 & 96.7 & 95.9 & 96.1 & 94.9 & 95.2 & 93.2 & 94.5 \\
    FashionMNIST & 80.2 & 82.9 & 80.1 & 83.1 & 79.7 & 82.4 & 78.8 & 80.8 & 76.2 & 79.4 \\
    \bottomrule
    \end{tabular}}
    \caption{Accuracy (\%) under different privacy‑budget splits
    $\epsilon_1{+}\epsilon_2{=}1.0$, fixing $\delta_1=\delta_2=5\times10^6$. Results show that allocating
    \,$\epsilon_1{:}\epsilon_2{=}\,0.8{:}0.2$\, offers a overall good trade‑off.}
    \label{tab:privacy_splits}
\end{table}

\subsection{SER Performance Across Varying Noise Levels \& Subspace Dimensions}
\label{sec:ser_noise}

Figure~\ref{fig:noise_multiplier_mse_mnist} details how the mean squared error (MSE) of mean estimation evolves on MNIST when varying both the \emph{noise multiplier} and the number of retained \emph{subspace dimensions} (horizontal axis in each subplot). Solid curves denote our method with SER (\emph{w/ SER}); dashed curves are the vanilla DP baseline (\emph{w/o SER}). A clear pattern, consistent with the residual decomposition in Theorem~\ref{thm:ser}, emerges:

\paragraph{Low-noise regime (noise multiplier $\lesssim 0.4 \times 10^{-3}$).}
Here, the DP noise injected per coordinate is small, so the total error is dominated by the \emph{projection error} introduced by compressing and reconstructing the data. In this regime, SER can even \emph{increase} MSE if the bottleneck is too tight; the loss of information outweighs the modest noise reduction. Consequently, retaining more subspace dimensions monotonically lowers the error, and the gap between ``w/'' and ``w/o'' SER narrows.

\paragraph{High-noise regime (noise multiplier $\gtrsim 0.9 \times 10^{-3}$).}
When the privacy budget is tight, the additive Gaussian noise dominates. Dimensionality reduction now acts as a signal-to-noise enhancer: a lower-rank subspace filters out much of the high-dimensional noise before reconstruction. As a result, SER yields a pronounced MSE drop relative to the baseline, particularly when only a few hundred components are kept. Beyond this point, adding more dimensions simply reintroduces noise and the benefit diminishes.

\paragraph{Intermediate-noise regime ($\sim 0.5 \times 10^{-3}$ to $0.8 \times 10^{-3}$).}
At moderate noise levels, the two error sources balance each other. The MSE curves adopt a classic U-shape, indicative of a trade-off: MSE first decreases as noise is tamed by projection, reaches a minimum at an \emph{optimal} dimensionality (typically 300–800 components), then increases again as projection bias begins to dominate. This turning point aligns with the crossover predicted by the dimensional-reduction effect term in Theorem~\ref{thm:ser}.

\medskip
Together, these three regimes offer actionable insight into how SER should be tuned in practice:  
\begin{itemize}
    \item When privacy is \textbf{loose}, favor a larger subspace or skip SER entirely.
    \item When privacy is \textbf{tight}, reduce dimensionality aggressively to suppress noise.
    \item For \textbf{intermediate} privacy budgets, select the number of subspace dimensions that minimizes MSE.
\end{itemize}

We apply this same strategy to FashionMNIST and CIFAR-10 (Figures~\ref{fig:noise_multiplier_mse_fmnist} and~\ref{fig:noise_multiplier_mse_cifar}), and observe analogous trends.

\begin{figure*}
    \centering
    \includegraphics[width=\linewidth]{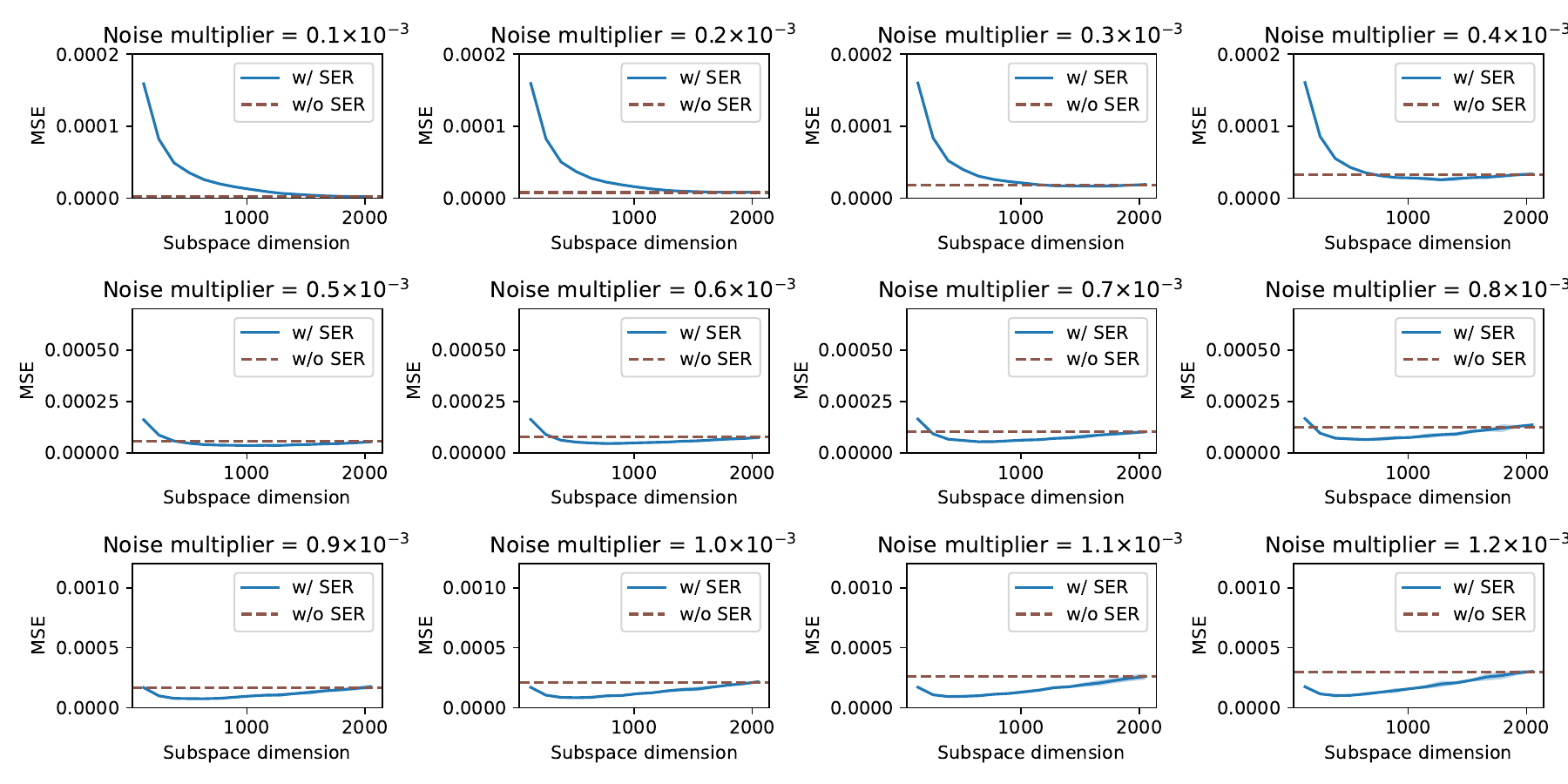}
    \caption{\textbf{MNIST: MSE of mean estimation as a function of retained subspace dimensions and noise multiplier.}
    Each subplot corresponds to a different \emph{noise multiplier} (privacy level).  
    The horizontal axis shows the number of retained subspace dimensions; the vertical axis shows mean‑squared error (MSE).  
    Solid curves are our method with \emph{SER}; dashed curves are the vanilla DP baseline.}
    \label{fig:noise_multiplier_mse_mnist}
\end{figure*}

\begin{figure*}
    \centering
    \includegraphics[width=\linewidth]{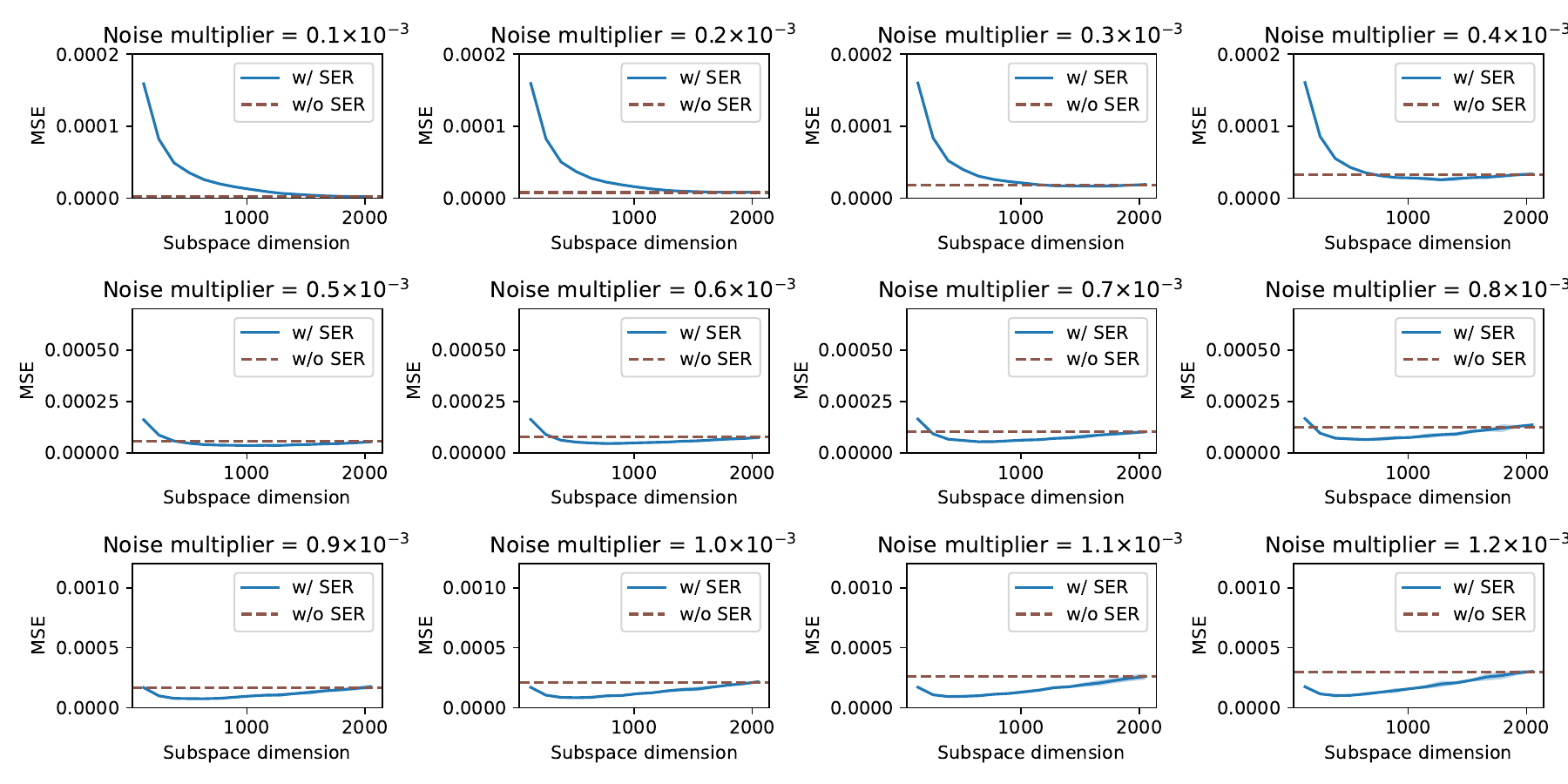}
    \caption{\textbf{FashionMNIST: MSE of mean estimation versus subspace dimension and noise multiplier.}
    Plot settings match Fig.~\ref{fig:noise_multiplier_mse_mnist}.  
    FashionMNIST exhibits the same qualitative behavior: SER offers little benefit in the low‑noise regime, achieves a clear optimum in the intermediate regime (300–800 components), and substantially reduces MSE under tight privacy budgets (high noise multipliers).}
    \label{fig:noise_multiplier_mse_fmnist}
\end{figure*}

\begin{figure*}
    \centering
    \includegraphics[width=\linewidth]{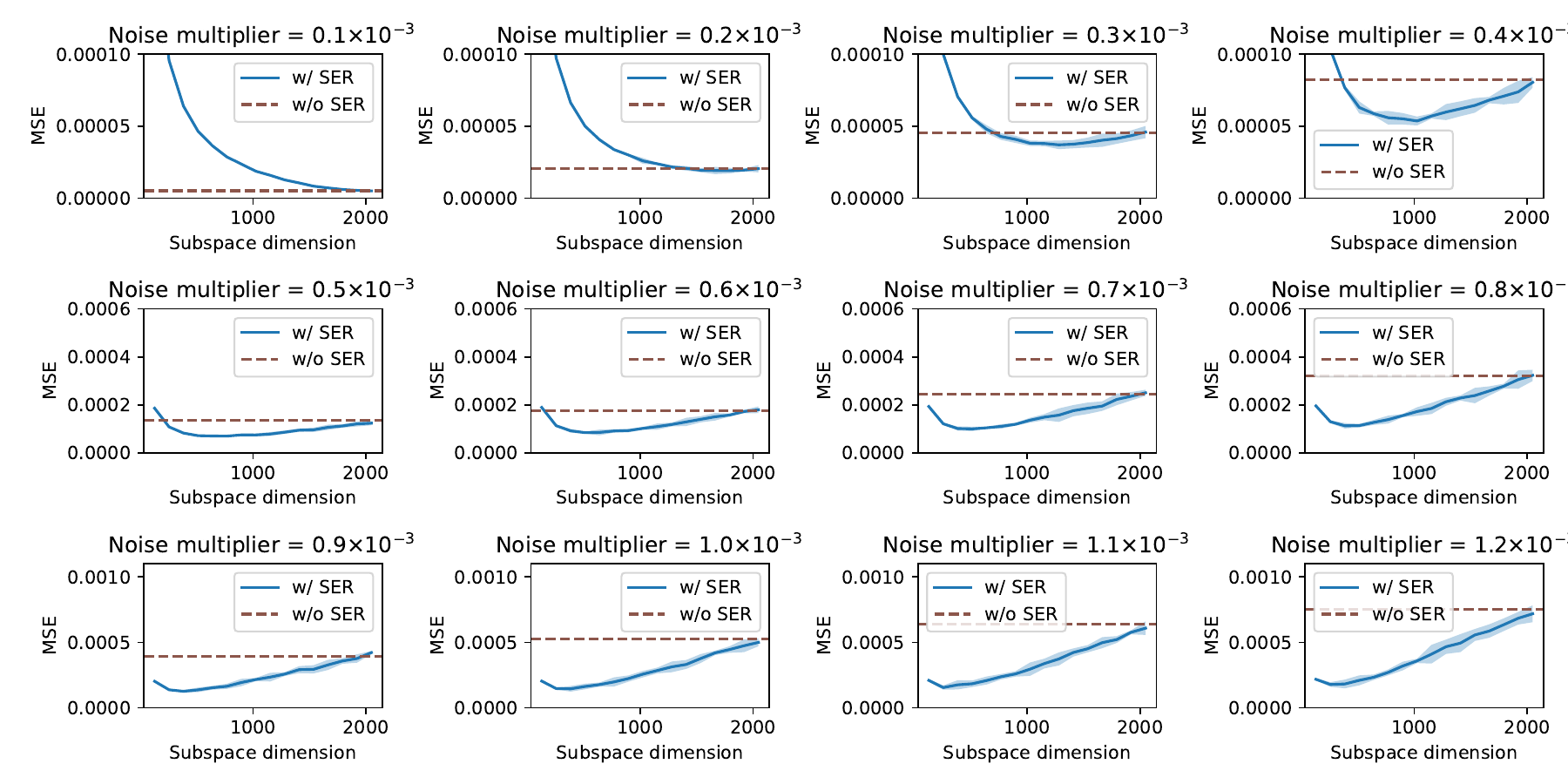}
    \caption{\textbf{CIFAR‑10: MSE of mean estimation versus subspace dimension and noise multiplier.}
    Despite the higher input dimensionality of CIFAR‑10, the same trends appear: SER markedly lowers the MSE when privacy is tight (high noise), has diminishing returns as more subspace dimensions are added, and converges to the baseline when privacy is loose.}
    \label{fig:noise_multiplier_mse_cifar}
\end{figure*}

\section{Qualitative Results}

In \cref{fig:samples}, we present distilled samples from the CIFAR-10, FashionMNIST, and MNIST datasets. Each row corresponds to a distinct class, with all samples generated using an IPC of 10 and a privacy budget of $(1, 10^{-5})$.

\begin{figure*}
    \centering
    \includegraphics[width=\linewidth]{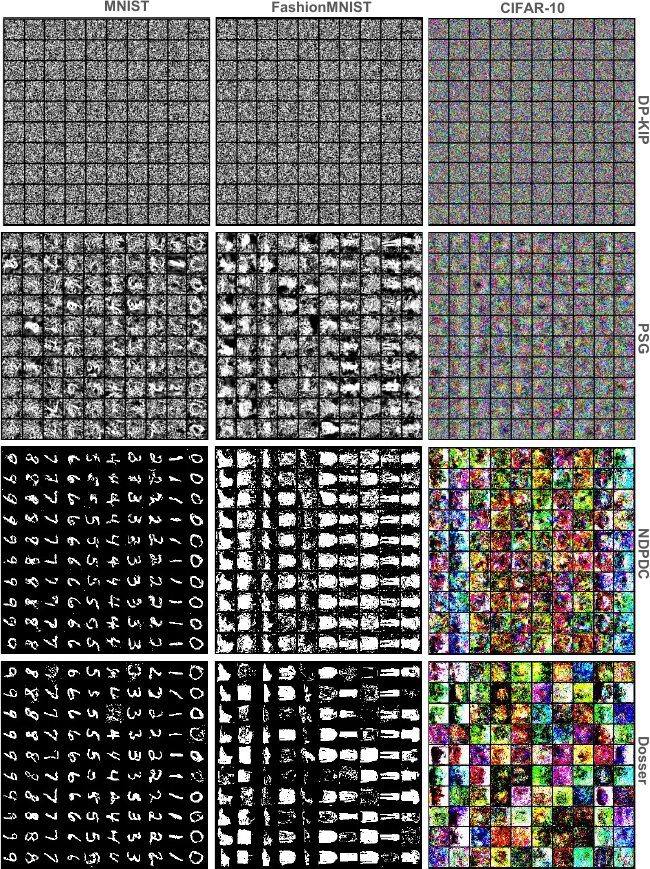}
    \caption{Distilled samples from the CIFAR-10, FashionMNIST, and MNIST datasets arranged in a 10$\times$10 grid. Each row represents a specific class, and all samples are generated with an IPC of 10 and a privacy budget of $(1, 10^{-5})$.}
    \label{fig:samples}
\end{figure*}

\section{Settings for Generating Auxiliary Datasets}
\label{sec:settings}

\subsection{Auxiliary Data Generation with Stable Diffusion (SD)~\cite{rombach2022high}}
For the CIFAR-10 dataset, we generate auxiliary images using Stable Diffusion version 1.4 (SD-v1-4). The generation process employs the following prompt for each category:
$$
\text{``A photo of a \{category\}''}.
$$
SD-v1-4 was trained on LAION-5B~\cite{schuhmann2022laion}, a dataset that contains no information related to CIFAR-10. Therefore, using it to train CIFAR-10 is not considered a privacy leakage. Representative image samples are illustrated in \cref{fig:aux_cifar10}.
\begin{figure*}[ht]
    \centering

    \begin{subfigure}[t]{\linewidth}
        \centering
        \includegraphics[width=\linewidth]{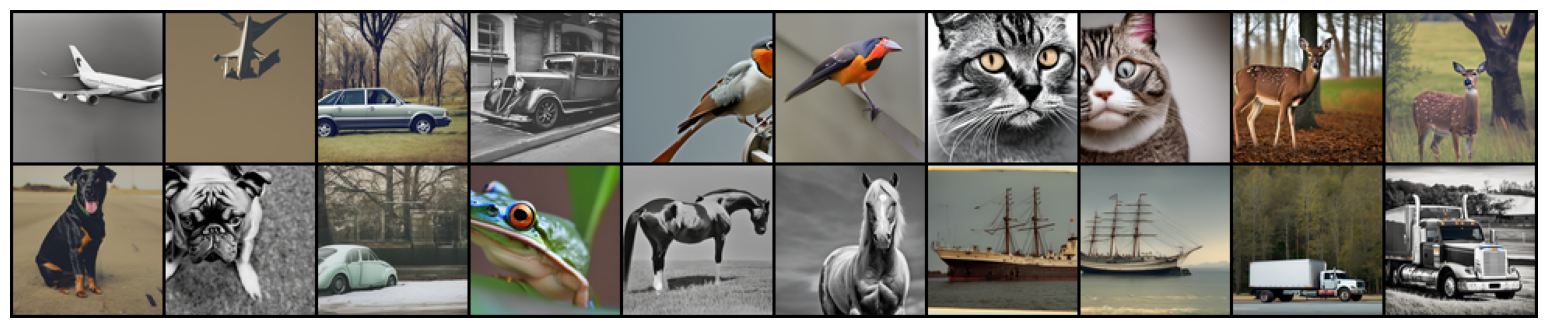}
        \caption{Sampled images from the CIFAR-10 auxiliary dataset.}
        \label{fig:aux_cifar10}
    \end{subfigure}

    \begin{subfigure}[t]{\linewidth}
        \centering
        \includegraphics[width=\linewidth]{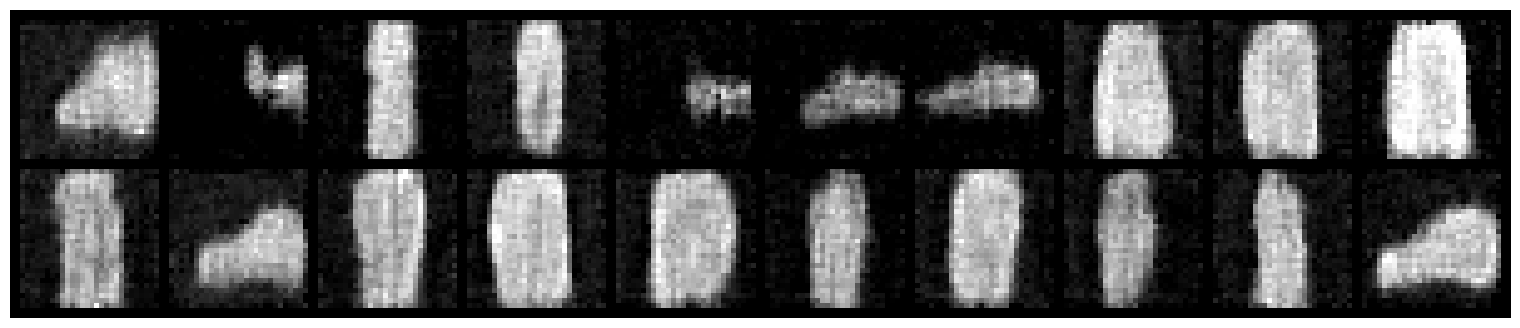}
        \caption{Sampled images from the FashionMNIST auxiliary dataset.}
        \label{fig:aux_fashionmnist}
    \end{subfigure}

    \begin{subfigure}[t]{\linewidth}
        \centering
        \includegraphics[width=\linewidth]{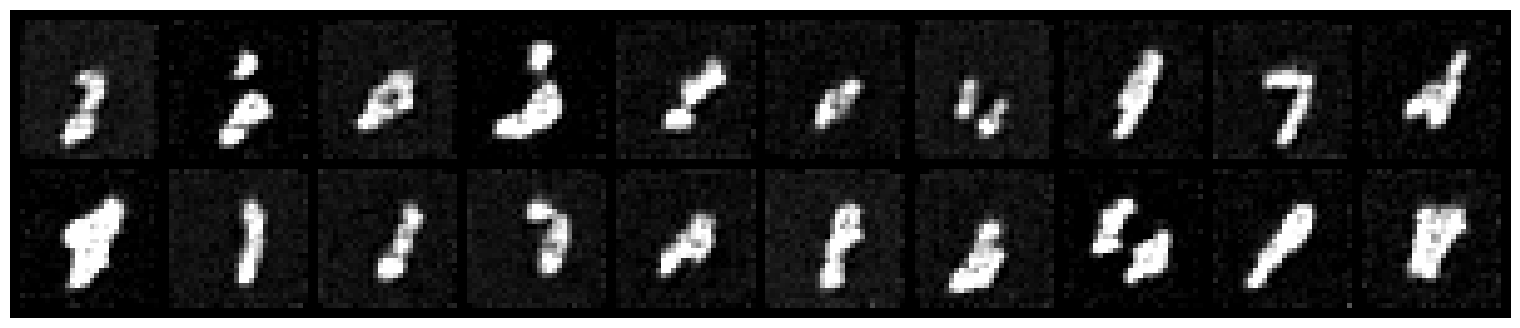}
        \caption{Sampled images from the MNIST auxiliary dataset.}
        \label{fig:aux_mnist}
    \end{subfigure}

    \caption{Sample auxiliary datasets.}
    \label{fig:combined_samples}
\end{figure*}

\subsection{Auxiliary Data Generation using Differentially Private Diffusion Model}
For MNIST and FashionMNIST we generate auxiliary images with the Differentially Private Diffusion Model (DPDM). Concretely, we train a Noise Conditional Score Network (NCSN++)\cite{song2021scorebased} for 50 epochs using Adam \cite{kingma2014adam} (no weight decay), a batch size of 64, and a learning rate of $3\times10^{-4}$. The trained network is then sampled with a deterministic DDIM sampler \cite{songdenoising} for 500 inference steps, ensuring the entire procedure conforms to the prescribed differential‑privacy budget.
We sample random images from the auxiliary dataset in \cref{fig:aux_fashionmnist} and \cref{fig:aux_mnist}.
\subsection{Other Models as Auxiliary Data Generator}
\label{sec:other_models}
Beyond SD and DPDM, we evaluate DP-Diffusion and DP-LDM as alternative generative models for producing auxiliary datasets under differential privacy constraints. In this experiment, we generate synthetic data for MNIST, FashionMNIST, and CIFAR-10 using each model while maintaining a fixed privacy budget $(1, 10^{-5})$. To assess the impact of dataset size, we vary the number of images per class (IPC) between 10 and 50. The generated datasets are then used to train downstream models, following the same evaluation protocol as in previous experiments. The results are shown in \Cref{tab:gen_model}.

\begin{table}[h]
    \centering
    \label{tab:gen_model}
    \setlength{\tabcolsep}{3pt}
    \caption{Comparison of DP-based generative models for SER.}
    \resizebox{\columnwidth}{!}{%
    \begin{tabular}{l|cc|cc|cc}
        \toprule
        & \multicolumn{2}{c|}{DPDM} & \multicolumn{2}{c|}{DP-Diffusion} & \multicolumn{2}{c}{DP-LDM} \\
        Dataset & IPC=10 & IPC=50 & IPC=10 & IPC=50 & IPC=10 & IPC=50 \\
        \midrule
        MNIST & \textbf{96.4} & 96.7 & 96.3 & 96.7 & 96.3 & \textbf{96.8} \\
        FashionMNIST & 80.1 & 83.1 & 80.5 & 83.0 & \textbf{80.8} & \textbf{83.4} \\
        CIFAR-10 & 47.8 & \textbf{51.5} & 47.5 & 51.0 & \textbf{48.2} & 51.2 \\
        \bottomrule
    \end{tabular}
    }
\end{table}

\subsection{Controlling for the Impact of Extra Information in DP-Based Generative Models}
To isolate the effect of additional information introduced by different generative models, we compare our method with DPDM, DP-Diffusion, and DP-LDM under the same privacy budget of $(1, 10^{-5})$. This comparison helps determine whether simply using these models for downstream training provides sufficient utility or if our method introduces meaningful improvements beyond what these baselines achieve. The results are shown in \Cref{tab:dp_baseline}.

\begin{table}[h]
    \centering
    \label{tab:dp_baseline}
    \setlength{\tabcolsep}{3pt}
    \caption{DP-based enerative models as baselines.}
    \resizebox{\columnwidth}{!}{%
    \begin{tabular}{l|cc|cc|cc|cc}
        \toprule
        & \multicolumn{2}{c|}{Dosser} & \multicolumn{2}{c|}{DPDM} & \multicolumn{2}{c|}{DP-Diffusion} & \multicolumn{2}{c}{DP-LDM} \\
        Dataset & IPC=10 & IPC=50 & IPC=10 & IPC=50 & IPC=10 & IPC=50 & IPC=10 & IPC=50 \\
        \midrule
        MNIST & \textbf{96.4} & \textbf{96.7} & 69.1 & 70.5 & 72.3 & 74.8 & 71.5 & 73.6 \\
        FashionMNIST & \textbf{80.1} & \textbf{83.1} & 59.7 & 63.6 & 60.2 & 65.1 & 61.1 & 64.9 \\
        CIFAR-10 & \textbf{50.6} & \textbf{52.3} & 10.0 & 9.9 & 10.0 & 10.0 & 10.0 & 10.2 \\
        \bottomrule
    \end{tabular}%
    }
\end{table}


\section{Discussion}

\subsection{Why Not DP-PCA for SER?}
\label{app:dppca}

One may argue that, instead of learning a fixed projection from auxiliary data, we could simply run DP–PCA at every iteration to discover the informative subspace on‑the‑fly. However, because the extractor is randomly re‑initialized each iteration, its output lies in a fresh feature space. Performing DP‑PCA on \emph{every} feature batch would therefore require an independent DP query each time. With a total budget $(\epsilon,\delta)$ split across $I$ iterations, each PCA call receives only $\epsilon/I$ privacy, which completely drowns the signal and devastates downstream accuracy.

\end{document}